\documentclass{article}
\usepackage{ijcai26}

\usepackage{times}
\usepackage{soul}
\usepackage{url}
\usepackage[hidelinks]{hyperref}
\usepackage[utf8]{inputenc}
\usepackage[small]{caption}
\usepackage{graphicx}
\usepackage{amsmath}
\usepackage{amsthm}
\usepackage{booktabs}
\usepackage{algorithm}
\usepackage{algorithmic}
\usepackage[switch]{lineno}

\usepackage{tikz}
\usepackage{amssymb,amsfonts}
\usepackage{xcolor}
\usepackage{natbib}

\newcommand{\X}{\ensuremath{\mathcal{X}}}
\newcommand{\F}{\ensuremath{\mathcal{F}}}
\newcommand{\Y}{\ensuremath{\mathcal{Y}}}
\newcommand{\D}{\ensuremath{\mathcal{D}}}

\DeclareMathOperator*{\argmax}{arg\,max}

\newtheorem{definition}{Definition}
\newtheorem{theorem}{Theorem}

\newtheorem{observation}{Observation}
\newtheorem{proposition}{Proposition}

\newtheorem{example}{Example}

\title{Axiomatic Choice}

\author{
    Submission Number: 4319
}

\author{
Ben Abramowitz$^{1,2}$
\and
Nicholas Mattei$^1$\and
\affiliations
$^1$Tulane University\\
$^2$San Diego State University\\
\emails
\{babramow, nsmattei\}@tulane.edu
}
\begin{document}

\maketitle

\begin{abstract}
People care about decision outcomes and how decisions get made, both when making decisions and reflecting on decisions. But formalizing the full range of normative concerns that drive decisions is an open challenge.
We introduce Axiomatic Choice as a framework for making and evaluating decisions based on formal normative statements about decisions. These statements, or \textit{axioms}, capture a wide array of desiderata, e.g., ethical constraints, beyond the typical treatment in Social Choice.
Using our model of axioms and decisions we define key properties and introduce a taxonomy of axioms which may be of general interest. We then use these properties and our taxonomy to define the Decision-Evaluation Paradox, formalize the concepts of transparency and deception in explaining and justifying decisions, and reveal the limits of existing methods using axioms to make decisions.
\end{abstract}

\section{Introduction}
When it comes to decisions, humans care about both decision outcomes and how decisions get made. We want decisions to be fair, efficient or democratic; for models and mechanisms to be transparent, auditable, or privacy-preserving; for rules to be simple or understandable; for outcomes to be equal or just; and for AI agents to act in alignment with human values. But to say precisely what each of these means, and ensure these goals are achieved, we need formalisms capable of capturing and operating on the wide range of normative positions held by human and non-human agents. To this end we introduce \textit{Axiomatic Choice}, a layer of abstraction above Social Choice.

There is an open call for formal and computational models to make the science of democratic decisions more rigorous~\citep{grossi2024enabling}.
Currently the primary tools for formalizing democratic decisions and norms come from Social Choice.
In Social Choice, multiple agents' preferences over outcomes are aggregated by some decision-making procedure or \textit{rule} that selects an outcome. Axioms, which are properties of rules, are used to characterize rules and to prove the incompatibility of certain desiderata~\citep{arrow2010handbook}.
%
% But what properties make a system or procedure democratic is a matter of debate~\citep{dahl1956preface,schumpeter2013capitalism}, and not all axioms are motivated by a desire for democracy.
%
% However, agents may have preferences over axioms and rules, not only outcomes, and these preferences may be relevant when making decisions and when reflecting on decisions, but difficult to handle.

We can split much of the active AI research on Social Choice axioms into those that use axioms to make decisions~\citep{matone2024deepvoting,schmidtlein2023voting,armstrong2019machine} and those that use axioms to evaluate, explain or justify decisions~\citep{boixel2022calculus,nardi2022graph,suryanarayana2022justifying,suryanarayana2022explainability,boixel2020automated,peters2020explainable,procaccia2019axioms,kirsten2018towards,cailloux2016arguing}. 
% These techniques are useful not only for collective decisions but also for decisions made by an individual agent. 
However, these approaches are currently limited by their focus on the narrow set of axioms that appear in Social Choice. Axiomatic Choice allows us to encode a broader range of realistic normative positions as axioms, make decisions according to these axioms, and generate explanations or justifications that agents will accept. It allows us to reason about aggregating agent preferences over decisions, rather than only preferences over outcomes.

Reinforcement learning from human feedback (RLHF) is a technique for training models to align with human goals and is commonly used to fine-tune LLMs. One of the key considerations for RLHF is the format in which human provides feedback~\citep{casper2023open}. If the domain of human feedback is restricted, say to ordinal preferences over a fixed set of candidates, then the feedback may not fully capture the normative concerns of the human, but if the feedback is unconstrained like natural language, then there is a problem of interpreting the feedback, e.g., whether and how to change the objective (reward function) or whether to add a new hard constraint on behavior. This balance between interpretability and expressiveness creates the need for formalisms capable of capturing a broad array of normative positions and using them for learning and applying trained models. There is the additional problem of conflicting feedback from multiple humans who hold a diverse set of values, and so methods of aggregating values are necessary. These concerns extend to any human-in-the-loop system concerned with value alignment~\citep{gabriel2020artificial,mosqueira2023human}, and it has been argued that tools from Social Choice can address these issues~\citep{noothigattu2020axioms,conitzer2024social,ge2024axioms}. We argue that Axiomatic Choice provides a more appropriate framework, and can model values that cannot easily be represented or differentiated by Social Choice.

\section{Model}

\paragraph{Decisions}
Let $\mathcal{X}$ and $\mathcal{Y}$ be non-empty sets.
Let $\mathcal{F}$ be exactly the set of all distinct, deterministic functions $f$ such that $f: \mathcal{X} \rightarrow \mathcal{Y}$.
Thus, $\mathcal{X}$ is the domain of all functions $f \in \mathcal{F}$, and for every pair of functions $f,f' \in \mathcal{F}$, $\exists x \in \mathcal{X}$ such that $f(x) \neq f'(x)$.\footnote{We define our rules to be deterministic, but the elements of $\mathcal{X}$ can contain random bits, thus encapsulating randomized rules.}

We denote by $\mathcal{D} = \{(x, f, y) : f(x) = y, f \in \mathcal{F}, x \in \mathcal{X}, y \in \mathcal{Y}\}$ the set of all \emph{decisions} with respect to $\mathcal{X}$ and $\mathcal{Y}$, where a decision corresponds to the application of a function $f$ to an element in its domain $x \in \mathcal{X}$ to produce an element of its codomain $y \in \mathcal{Y}$. If $f(x) \neq y$, then $(x, f, y) \not \in \mathcal{D}$ and $(x, f, y)$ is not a decision.

These are our initial building blocks for modeling axiomatic choice. We refer to the elements of $\mathcal{X}$ as profiles, the elements of $\mathcal{Y}$ as outcomes, the functions in $\mathcal{F}$ as rules, and the tuples in $\mathcal{D}$ as decisions. We can use a rule to make a decision, which consists in taking in a profile and producing an outcome.
Notice that the two variables $(x, f)$ are enough to uniquely determine $y = f(x)$, so the tuple $(x, f, y)$ contains redundant information. We do this for clarity.

\begin{figure}[h!]
\begin{center}
\scalebox{1}{
\begin{tikzpicture}
  \node[circle,draw] (outcome){outcome} [grow'=up]
    child {node[circle,draw] (profile){profile}}
    child {node[circle,draw] (rule){rule}
    };
    \draw[dashed,->] (profile) -- (outcome);
    \draw[dashed,->] (rule) -- (outcome);
\end{tikzpicture}
}
\end{center}
\caption{Components of a decision}
\end{figure}
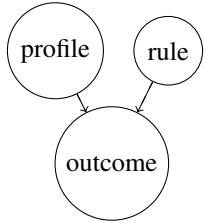

\paragraph{Axioms}
An axiom $A$ induces a partition $P_A:\mathcal{D} \rightarrow \{0,1\}$ on a set of decisions $\mathcal{D}$, splitting it into those decisions that \emph{violate} the axiom (0) and a complementary set that \emph{obey} the axiom (1). We abuse notation slightly and use functional notation where $A(x,f,y) = 1$ means decision $(x,f,y)$ obeys $A$, while $A(x,f,y)=0$ means the axiom is violated. Any axiom can be uniquely characterized by the set of all decisions $L \subseteq \D$ that obey it.

Axioms are used to make and evaluate decisions. We say that we \textit{evaluate a decision} when we check whether it obeys a given axiom. When used to \textit{make decisions}, an axiom restricts the set of allowed rules and possible outcomes given any profile. We will refer to the following running example throughout the paper. 

\begin{example}[Running Example]\label{ex:running}
Let $\mathcal{X} = \{x_1, x_2\}$, $\mathcal{Y} = \{y_1, y_2\}$, and $\mathcal{F}$ be defined as follows:

$$\mathcal{F} =\begin{cases} 
    f_1 := f_1(x_1) = y_1, & f_1(x_2) = y_1 \\
    f_2 := f_2(x_1) = y_1, & f_2(x_2) = y_2 \\
    f_3 := f_3(x_1) = y_2, & f_3(x_2) = y_1 \\
    f_4 := f_4(x_1) = y_2, & f_4(x_2) = y_2 \\
\end{cases}$$

$$L =\begin{cases}
    (x_1, f_1, y_1) \\
    (x_2, f_4, y_2) \\
\end{cases}$$
\end{example}

% For example, we may evaluate many prospective decisions when making a decision to find one that obeys a desired axiom.

\paragraph{Generality}
Across all of Social Choice, $\X$ is related to $\Y$ in some way, e.g., $\Y$ may be a set of candidates from which a single winner must be chosen and $\X$ may be the set of possible orderings over the candidates in $\Y$.
By contrast, we impose no particular relationship between the elements of $\mathcal{X}$ and elements of $\mathcal{Y}$. 
In fact, we specify nothing about the elements of $\mathcal{X}$ or $\mathcal{Y}$ whatsoever. They can be any sort of objects: sets, functions, binary strings, natural language statements, partial orders, directed graphs, probability distributions, grandmothers, etc.
Our model is entirely domain-agnostic and applies equally well to individual and collective decisions.

\subsection{Rules vs. Profile-Outcome Pairs}\label{sec:rules_vs_lists}
It is informative to begin by illustrating the difference between two simple axioms.
Let $B \subseteq \X \times \Y$ be a set of profile-outcome pairs that includes a single pair for every profile $x \in \mathcal{X}$. In other words, $B$ maps each profile to a single outcome.
The set $B$ can be used to define a corresponding rule $f_B: \mathcal{X} \rightarrow \mathcal{Y}$ where $f_B(x) = y$ if and only if $(x,y) \in B$. Notice that for any such set $B$ we can find a corresponding rule $f_B$, and that for any rule $f_B \in \mathcal{F}$ we can define the corresponding set $B$. Given a profile $x$, if we make decisions by looking up the corresponding entry in $B$ and outputting the outcome $y$, then we are implicitly implementing the rule $f_B$.

Now compare the following two axioms. Define axiom $A_B$ such that $A_B(x,f,y)=1$ if and only if $(x,y) \in B$, and define axiom $A_F$ such that $A_F(x, f, y) = 1$ if and only if $f = f_B$. 
Clearly, if we use $f_B$ to make any decision $(x, f_B, f_B(x))$ then this decision must obey both $A_B$ and $A_F$. However, $A_B$ and $A_F$ are not the same axiom! 
Let $(x,g,y)$ be a decision such that $(x,y) \in B$ but $g \neq f_B$. The decision $(x,g,y)$ will obey $A_B$ because $(x,y) \in B$, but it will violate $A_F$ because $g \neq f_B$.
$A_B$ and $A_F$ induce different partitions of the set of decisions $\mathcal{D} = \mathcal{X} \times \mathcal{F} \times \mathcal{Y}$. The set of decisions that obey $A_F$ is a subset of those that obey $A_B$.

Using Example~\ref{ex:running}, if $B = \{(x_1,y_1),(x_2,y_2)\}$ then $f_B = f_2$, and the decision $(x_1, f_1, y_1)$ will obey axiom $A_B$ but not $A_F$ while $(x_1, f_2, y_1)$ and $(x_2, f_2, y_2)$ will obey both axioms.

% To determine whether a decision $(x,\underline{\hspace{4mm}},y)$ obeys axiom $A_B$, we only need to know the profile $x$ and outcome $y$ for this single decision. But to determine whether a decision $(x,g,y)$ obeys $A_F$, we need to know the rule $g$.

As a concrete example, consider the difference between ``The rule must be Pareto-efficient" and ``The outcome must be Pareto-optimal." When these axioms are used to evaluate a decision they can arrive at different results, because rules may select Pareto-efficient outcomes for some but not all profiles. But when used to make decisions, both axioms enforce a Pareto-optimal outcome for every profile. 
% Our paradox will examine the case where the opposite occurs -- where making a decision using an axiom leads to a decision that violates the same axiom upon evaluation.

\subsection{Contributions}
Using our domain-agnostic framework of Axiomatic Choice we demonstrate the fundamental challenges of using normative statements encoded as formal axioms to make decisions and evaluate decisions. We (1) develop a new taxonomy of axioms based on their structural properties, (2) formalize the concepts of transparency and deception central to explaining and justifying decisions, (3) reveal how the Decision-Evaluation Paradox creates a tension between the use of axioms when making and evaluating decisions, (4) provide a new minimal, domain-agnostic, language-independent definition of intraprofile axioms using our constructs, and (5) show how Decision-Evaluation paradox and deception limit axiom-based decision making procedures like those of \cite{schmidtlein2023voting} and \cite{dietrich2005reach}.

\section{Axiom Taxonomy}\label{sec_taxonomy}
We introduce a classification of axioms based on how they partition the set of possible decisions as shown in Table~\ref{tab_set_characterization}.
% This taxonomy is equivalent to classifying axioms based on which of the components ($x$, $f$, or $y$) need to be known to evaluate decisions.

\begin{table}[h!]
\begin{center}
\small
\begin{tabular}{|p{2.28cm}|p{5.72cm}|}
\hline
\textbf{Class} & \textbf{Definition} \\
\hline
Positively Trivial & $\forall (x,f,y) \in \D$ : $A(x,f,y)=1$ \\
\hline
Negatively Trivial & $\forall (x,f,y) \in \D$ : $A(x,f,y)=0$ \\
\hline
Structural & $\exists X \subseteq \X$ : $A(x,f,y)=1 \Leftrightarrow x \in X$ \\
\hline
Procedural & $\exists F \subseteq \F$ : $A(x,f,y)=1 \Leftrightarrow f \in F$ \\
\hline
Consequentialist & $\exists Y \subseteq \mathcal{Y}$ : $A(x,f,y)=1 \Leftrightarrow y \in Y$ \\
\hline
Blackbox & $\exists B \subseteq \mathcal{X} \times \mathcal{Y}$ : $A(x,f,y)=1 \Leftrightarrow (x,y) \in B$ \\
\hline
Caudal & $\exists C \subseteq \mathcal{F} \times \mathcal{Y}$ : $A(x,f,y)=1 \Leftrightarrow (f,y) \in C$ \\
\hline
Exigent & None of the Above \\
\hline
\end{tabular}
\end{center}
\caption{Classifying axioms by partition structure}
\label{tab_set_characterization}
\end{table}

An axiom $A$ is \textit{trivial} with respect to decision domain $\mathcal{D}$ if it is obeyed by all decisions in $\mathcal{D}$ or violated by all decisions in $\mathcal{D}$. 
%
% For a trivial axiom, we never need any information about a decision to determine if it obeys or violates the axiom.
%
We further divide trivial axioms into positively trivial ones that are obeyed by every decision, and negative trivial axioms that are violated by every decision. 
The axiom of ``Universal Domain" in Arrow's Impossibility Theorem is a prime example of a positively trivial axiom. Universal Domain says that $f(x) \in \Y$ for all $x \in \X$, which is implied by our definition of a rule and construction of $\F$. 
%
% On the other hand, classic impossibility theorems in Social Choice consist of constructing a negatively trivial axiom from a set of non-trivial axioms where the intersection of the sets of decisions that obey each axiom is empty.

Structural, procedural, and consequentialist axioms are the three most basic types of non-trivial axioms. 
% Each axiom partitions $\mathcal{D}$ with respect to a single decision component ($x$, $f$, or $y$).
%
Structural axioms characterize properties of profiles.\footnote{The term ``structural" is taken from \citet{fishburn2015theory}.} For example, a quorum requirement for passing legislation is a structural axiom and so is single-peakedness.
Procedural axioms characterize properties of rules, e.g., whether the rule is Pareto-efficient, strategyproof, or can be computed efficiently.
Consequentialist axioms characterize properties of outcomes without respect to the profile or rule, e.g., whether a winning candidate belongs to a particular political party. A person whose decision-making preferences are characterized by a consequentialist axiom cares only about the outcome and not how the decision gets made.

Blackbox and caudal axioms partition $\D$ with respect to two components of the decision.
Blackbox axioms can be expressed as sets of profile-outcome $(x,y)$ pairs.
For example, the statement ``The outcome must be Pareto-optimal" is a blackbox axiom, which differs from the procedural axiom, ``The rule must be Pareto-efficient."
Structural and consequentialist axioms are edge cases of blackbox axioms.
% %
% The set $L$ we saw in Section \ref{sec:rules_vs_lists} is a basic case where each profile $x$ only appears once, but blackbox axioms may require multiple entries for each profile $x$.
%
When it comes to notions of fairness and efficiency, ex-ante notions (e.g., anonymity) are typically procedural and ex-post notions are typically consequentialist or blackbox (e.g., equity or envy-freeness).

A caudal axiom depends on the rule and/or outcome but not directly on the profile, e.g., ``We should use Plurality rule to vote on the best paper at the conference, unless my paper wins, in which case the rule doesn't matter." Lastly, an axiom is exigent if it is not any of the other types listed above.

% The last combination of variables from which to construct an axiom class would be the profile and rule together. In other words,  $A(x, f, y) = A(x, f, y')$ for all such pairs of decisions. This does not restrict the set of axioms at all because $(x, f, y)$ and $(x, f, y')$ cannot both be decisions unless $y = y'$.
%
% Axiomatic approaches to AI Alignment and Social Choice focus largely on procedural and blackbox axioms.
%
We do not make any judgments here about what axioms are good or desirable or useful, but argue only that for each of our taxonomic classes there agents in the real world whose normative positions on some decisions belong to that class.

Finally, notice that using our taxonomy we can characterize any non-trivial non-exigent axiom by a single set of decision components $(X,F,Y,B, \text{ or } C)$, rather than the full set of decisions $L \subseteq \D$ that obey it. We will use both characterizations as convenient.

\section{Axiom Properties and Relations}
We now move to defining the various properties and relations of axioms that we need to develop our core results. We say that an axiom $A$ is \textit{implementable} if it does not allow an impasse, so for all profiles there is at least one decision that obeys the axiom.

\begin{definition}[Impasse]\label{def:impossible2}
    Given axiom $A$, an impasse is a profile $x \in \mathcal{X}$ such that $A(x, f, f(x)) = 0$ for all $f \in \mathcal{F}$.
\end{definition}

Expanding the definitions from \citet{schmidtlein2023voting}, we say that an axiom $A$ is \textit{forcing with respect to profile $x$} if there exists a single outcome $\textbf{y} \in \mathcal{Y}$ such that $A(x,f,y)=1$ if and only if $y=\textbf{y}$. We will further say that $A$ forces the outcome $\textbf{y}$ on $x$, and that the axiom itself is \textit{forcing} if it is forcing with respect to every profile in $\mathcal{X}$.

\begin{definition}[Forcing]
    Axiom $A$ is forcing if for all profiles $x \in \X$ there is a single outcome $\textbf{y} \in \Y$ such that $A(x,f,y)=1$ only if $y=\textbf{y}$. 
\end{definition}

We can use forcing axioms to make decisions because they imply a specific outcome for every profile and do not allow an impasse. A forcing axiom corresponds to a particular rule, just like $B$ and $f_B$ in Section~\ref{sec:rules_vs_lists}. We aim to clarify this relationship between forcing axioms and rules here.
Given the set of decisions $L \subseteq \D$ characterizing any axiom we can derive a blackbox axiom by reduction.

\begin{definition}[Blackbox Reduction]
    Given $L \subseteq \D$ characterizing axiom $A_L$, let $B = \{(x,y) : \exists f \in \mathcal{F} \text{ such that } (x,f,y) \in L\}$. We say that axiom $A_B$ characterized by $B$ is the blackbox reduction of $A_L$.
\end{definition}

\begin{definition}[Extensional Equivalence]
    Two axioms are extensionally equivalent if they reduce to the same blackbox axiom.
\end{definition}

Many axioms may reduce to the same blackbox axiom, but every axiom has a single reduction. From any blackbox axiom we can construct a unique procedural axiom by \textit{extension}.

\begin{definition}[Procedural Extension]
    Let $A$ be any axiom that reduces to the blackbox axiom $A_B$ characterized by $B \subseteq \X \times \Y$, and let $F = \{f \in \F : \forall x \in \X, (x,f(x)) \in B\}$. We call the procedural axiom characterized by $F$ the procedural extension of $A$.
\end{definition}

Every blackbox axiom has a single procedural extension, but all blackbox axioms that allow an impasse their extension is negatively trivial. For implementable blackbox axioms, procedural extension generalizes the ``interpretation of outcome statements" from \cite{boixel2022calculus}.

Finally, before introducing our main theoretical results, we must bring extensional equivalence and forcing together to define the concept of an implied rule, which generalizes the ``induced voting rule" of \cite{schmidtlein2023voting}.

\begin{definition}[Implied Rule]
    Axiom $A$ implies rule $f$ if $A$ is extensionally equivalent to the procedural axiom characterized by $F = \{f\}$.
\end{definition}

We now have the basic tools necessary to examine the role of axioms in making and evaluating decisions. We begin with a few basic observations, omitting proofs for brevity, before moving to our main results.

\begin{observation}
    The blackbox reduction and procedural extension of an axiom are extensionally equivalent to the original axiom.
\end{observation}

\begin{observation}
    If an axiom is forcing then so is every axiom extensionally equivalent to it.
\end{observation}

\begin{observation}
    The procedural extension of an axiom is characterized by a single rule $F = \{f\}$ if and only if the axiom is forcing, so all forcing axioms imply a rule.
\end{observation}

\section{The Decision-Evaluation Paradox}\label{sec_paradox}
We can commit to obeying a forcing axiom when making a decision before ever observing a profile, knowing that it will lead us to a unique outcome. The process of making a decision guided by the forcing axiom implements the axiom's implied rule. Counterintuitively, if we take the rule implied by a forcing axiom and use it to make decisions, those decisions are not guaranteed to obey the axiom. We call this the Decision-Evaluation Paradox. Below is a minimal example proving the existence of the paradox.

\begin{theorem}\label{thm_implied_rule_not_obey}
    Making a decision using the implied rule of a forcing axiom does not guarantee that the decision obeys the axiom.
\end{theorem}

\begin{example}[Decision-Evaluation Paradox]\label{ex:paradox}
Recall Example~\ref{ex:running} with $L = \{(x_1,f_1,y_1),(x_2,f_4,y_2)\}$, and let $A_L$ be the forcing exigent axiom characterized by $L$. The reduction of $A_L$ yields $A_B$ characterized by $B = \{(x_1, y_1),(x_2,y_2)\}$. Axiom $A_B$ implies the rule $f_2$ because $f_2(x_1) = y_1$ and $f_2(x_2) = y_2$, and therefore $A_L$ implies $f_2$ as well. However, no decision made using $f_2$ will ever obey $A_L$, even though it obeys $A_B$.
\end{example}

The Decision-Evaluation Paradox shows us that using axioms to make decisions does not guarantee that our decisions will obey those axioms. However, this paradox only arises for caudal and exigent axioms.

\begin{theorem}\label{prop_forcing_BB_or_F}
    Making decisions using the implied rule of a forcing blackbox axiom or forcing procedural axiom guarantees that the decision obeys the axiom.
\end{theorem}

\begin{proof}
    Every forcing blackbox axiom $A_B$ has a unique procedural extension which is a single rule that is its implied rule $f_B$, and therefore $f_B(x) = y \Rightarrow A_B(x,f_B,y)=1$. For any forcing procedural axiom $A$, there is a single rule $\textbf{f}$ such that $A(x,f,y)=1$ if and only if $f = \textbf{f}$, which must also be $A$'s implied rule.
\end{proof}

While the paradox does not appear with blackbox or procedural axioms, that does not mean we can ignore it. For example, if we care about both ex-ante and ex-post fairness criteria, then this combination may be exigent. Generally, combining blackbox with procedural axioms can yield caudal and exigent axioms, and we will clarify what we mean by combining axioms in Section~\ref{sec:set_operations}.
It remains to show precisely when this paradox arises for caudal and exigent axioms. To this end, we need the classical notion of an Arrovian impossibility from Social Choice.

\begin{definition}[Arrovian Impossibility]\label{def:impossible1}
    Axiom $A$ is an Arrovian impossibility if $\nexists f \in \mathcal{F}$ such that for all $x \in \mathcal{X}$, $A(x,f,f(x))=1$.
\end{definition}
 
% We can see that if a procedural axiom is an Arrovian impossibility then it is negatively trivial. 
Example \ref{ex:paradox} shows how the paradox arises when an axiom is forcing but is still an Arrovian impossibility. The paradox occurs whenever an Arrovian impossibility is implementable, so it does not allow an impasse (but does not need to be forcing). In this case none of the rules in the set $F$ that characterizes its procedural extension can make decisions that obey the axiom for all profiles. The decision-evaluation paradox is of particular concern when defining procedures that explicitly combine multiple axioms to make decisions, like Voting by Axioms.

\subsection{Voting by Axioms}\label{sec:voting_by_axioms}
\cite{schmidtlein2023voting} proposed a process called Voting by Axioms for combining procedural axioms to make decisions. Here we generalize it to all possible axioms. Construct an ordered set of axioms $\vec{A} = (A_1, \ldots, A_n)$. For each axiom $i \leq n$, let $X_i \subseteq \X$ be the subset of profiles with respect to which $A_i$ is forcing and for which $A_j$ is not forcing for any $j < i$. Given any profile $x$ such that $x \in X_i$, select the outcome forced by $A_i$. If $\cup_i X_i = \X$, then the Voting by Axioms procedure selects a unique outcome for every profile, and this can be guaranteed by requiring $A_n$ be forcing so that $X_n = \mathcal{X}$.

In our framework, Voting by Axioms can be seen as one method among many to combine multiple axioms into a single forcing axiom using set operations on their partitions of $\D$, and hence a unique implied rule. Let $\tilde{L}_i  = \{(x,f,y) \in \D : x \in X_i \text{ and } A_i(x,f,y) = 1\}$, let $\tilde{L} = \cup_i \tilde{L}_i$, and let $A_L$ be the forcing axiom characterized by $\tilde{L}$.  $A_L$ could belong to any of our taxonomic classes, and is therefore vulnerable to the decision-evaluation paradox, which arises exactly when $A_L$ is an Arrovian impossibility but does not allow an impasse.

If every axiom in $\vec{A}$ is blackbox, then $A_L$ is a forcing blackbox axiom and the decision-evaluation paradox cannot occur. By contrast, if the axioms in $\vec{A}$ are all procedural, $A_L$ is not guaranteed to be procedural, and the paradox can arise.
For example suppose $\vec{A} = (A_1,A_2)$ where the axioms are both procedural and characterized by $F_1 = \{f_1,f_2\}$ and $F_2 = \{f_3\}$ using the rules defined in Example~\ref{ex:running}. Then $\tilde{L} = \{(x_1,f_1,y_1),(x_1,f_2,y_1),(x_2,f_3,y_1)\}$ and $A_L$ is a forcing exigent axiom with the implied rule $f_1$ though making decisions with $f_1$ on profile $x_2$ will violate $A_L$. By contrast, if instead $F_2 = \{f_1, f_3\}$, then $A_L$ is still an exigent forcing axiom, but the decision-evaluation paradox does not occur.

%Asking when it obeys A_L versus obeying every axiom in \vec{A}. Are these the same thing? No...

\section{Set Operations}\label{sec:set_operations}
Axioms can be created from other axioms using set operations on their partitions of $\D$.
Given any set of axioms $\bar{A} = \{A_1, A_2, \ldots\}$ we can define $L_i = \{(x,f,y) \in \D : A_i(x,f,y)=1\}$ and construct $L = \cap_i L_i$ to characterize $A_L$. We will denote this by $A_L = AND(\bar{A})$. When none of the axioms $A_i$ are negatively trivial but $A_L$ is negatively trivial, this constitutes an Arrovian impossibility theorem.
Existing approaches to justifying decisions using axioms often consist of constructing an axiom $A_L$ from a set of procedural axioms $\bar{A}$ in this way and then demonstrating that $A_L$ is forcing for a given profile~\citep{boixel2020automated,boixel2022calculus,nardi2022graph}.

Axioms may be constructed from set operations in other ways. For example, we can use $L = \cup_i L_i$ to characterize $A_L$, where a decision obeys $A_L$ if it obeys any of the axioms in $A_i$, denoted by $A_L = OR(\bar{A})$. This may be useful when an agent may accept multiple different types of justification, e.g., ``Either the rule must be ex-ante fair or the outcome must be ex-post fair." These constructions make it easier to capture what it means to be reasonable, flexible, or accommodating when making and evaluating decisions.
 Voting by Axioms in Section~\ref{sec:voting_by_axioms} uses yet another construction for creating forcing axioms using set operations.

 \begin{proposition}[Taxonomic Closure]\label{prop:closure}
     If $\bar{A}$ is blackbox (resp. caudal) then $AND(\bar{A})$ and $OR(\bar{A})$ are blackbox (resp. caudal), but may be trivial.
 \end{proposition}

Proposition~\ref{prop:closure} is easily proven by examining set operations on the characterizing sets $B$ and $C$ rather than the full set of obeying decisions $L$; e.g., for all $B_1,B_2 \subseteq \X \times \Y$, $B_1 \cap B_2 \subseteq \X \times \Y$ and $B_1 \cup B_2 \subseteq \X \times \Y$.
 
The ability to construct complex axioms from sets of multiple simpler axioms is important for the elicitation of agents' normative positions, because complex axioms may be hard or unnatural to describe directly. For example, it is much more natural to think of Arrow's General Impossibility Theorem in terms of multiple non-trivial axioms rather than a single negatively trivial axiom.

Whatever means we use to combine axioms there is the risk that their combination allows an impasse or creates an Arrovian impossibility. In these cases, we might try to make decisions that obey the maximum number of axioms in $\bar{A}$, which we will revisit in Section~\ref{sec:max_acceptance}.

\section{Transparency and Deception}
%These paradoxes mean we may be mislead when evaluating a decision. We may be incorrect about whether a decision violates or obeys our chosen axiom.

We can define general concepts of transparency and deception from the perspective of an agent evaluating a decision. We assume that an agent's normative position is characterized by an axiom, and they ``accept" a decision if it obeys their axiom.

\begin{definition}[Transparency]
    A decision is transparent to an agent whose normative position is described by axiom $A$ if the agent has access to sufficient information about the decision to determine if it obeys or violates $A$.
\end{definition}

Our axiom taxonomy breaks up axioms according to which decision components ($x$, $f$, or $y$) are sufficient to evaluate an axiom $A$; for a procedural axiom $f$ is sufficient to determine $A(\_\_, f, \_\_)$, for a blackbox axiom $(x,y)$ is sufficient to determine $A(x, \_\_, y)$, for an exigent axiom all three components are needed, etc. Thus the class to which an axiom belongs tells us the kind of information that makes a decision transparent to an agent. A lack of transparency creates room for deception, which may be material or immaterial.

\begin{definition}[Deception]
    A statement about a decision $(x,f,y)$ is deceptive if the statement contains information implying that the decision is $(x',f',y') \neq (x,f,y)$. To an agent whose normative position is axiom $A$, this deception is said to be material to the agent if $A(x,f,y) \neq A(x',f',y')$, and immaterial otherwise.
\end{definition}

It is common for $x$ and $y$ to be observable but not the rule $f$, because one cannot be sure what would have happened if the profile were different unless the implementation of the rule is auditable. Any statement that implies a different rule was used contains false counterfactual implications about what would have happened had the profile been different, and creates deception. If the agent is told that a decision was $(x,f',y)$ when in reality it was $(x, f, y)$, then the agent is deceived. This deception will be immaterial if the agent's axiom is blackbox, but may be material otherwise. We refer to this as \textit{counterfactual deception}, and observe that failure to recognize the decision-evaluation paradox leads to counterfactual deception.

Consider the following exigent forcing axiom whose extension is characterized by Black's Rule: ``If a Condorcet-winner exists then use the Copeland Rule and otherwise use the Borda Rule". If we make decisions according to this axiom, then we are never truly using the Borda Rule or the Copeland Rule; we are using an implied rule that selects the Copeland winner on a subset of profiles and the Borda winner on others. Now suppose there is a stakeholder who is adamant that the decision should be made by the Borda Rule. Their normative position is a procedural axiom. For a given decision in which the winner is the Borda winner and there is no Condorcet-winner, will that participant be satisfied? If you explain to them that they should be satisfied because the ``Borda Rule" was used to select the winner, is that true or are you misleading them? If you are misleading them is it material? This explanation does materially deceive them because their axiom is procedural and allows only the Borda Rule on all profiles. If their axiom were a blackbox axiom that only required the Borda winner to win, then the same statement would not be deceptive, because the agent only needs to know the profile and outcome to evaluate the decision. The issue of deception arises when we try to provide a single explanation or justification for a decision, and comes into sharp relief when the agents are provided individualized explanations for the same decision.

\subsection{Maximizing Acceptance}\label{sec:max_acceptance}
Suppose we have a list of axioms $A^* = (A_1, \ldots, A_n)$, and we wish to make decisions that satisfy as many of the axioms as possible. For example, the axioms may encode the normative positions of $n$ agents involved in a collective decision who each see a decision as acceptable if only if the decision obeys their personal axiom. We want to make decisions that are acceptable to the maximum number of agents. This problem arises in the design of decision-making processes, including automated mechanism design~\citep{sandholm2003automated}, the founding and amendment of constitutions~\citep{abramowitz2021deciding}, and Constitutional AI~\citep{bai2022constitutional}.

If all axioms are procedural, they can be characterized by sets of rules $(F_1, \ldots, F_n) \subseteq \mathcal{F}^n$. To maximize acceptance, we make decisions by the rule: $f^* = \argmax\limits_{f \in \F} \sum\limits_{i\leq n} |\{f \cap F_i\}|$.\footnote{Ties broken arbitrarily} We are able to define the rule $f^*$ that will make decisions to maximize acceptance for any given profile, before ever observing a profile.

If all axioms are blackbox, we can characterize them by $(B_1, \ldots, B_n) \subseteq (\X \times \Y)^n$. For all profiles $x \in \X$, define $y_x = \argmax\limits_{y \in \Y} \sum\limits_{i \leq n} |\{(x,y) \cap B_i\}|$.\textsuperscript{3} Now define $f^*$ such that $f^*(x) = y_x$ for all $x \in \X$. Again, we are able to define a rule $f^*$ for making decisions that will maximize acceptance for any given profile, before observing a profile.

What we want is a mechanism that takes in any set of axioms and constructs a rule to makes decisions that maximize acceptance on all profiles. However, if the axioms can be caudal, exigent, or a mix of blackbox and procedural, then such a rule may not exist. Recall the definitions in Example~\ref{ex:running}. Suppose we have three agents with one blackbox axiom and two procedural axioms characterized by $B_1 = \{(x_1,y_1), (x_2,y_2)\}$, $F_2 = \{f_1\}$, and $F_3 = \{f_4\}$. Now there is no single rule that maximizes acceptance for all profiles.

It is tempting then to construct a mechanism that takes in $A^*$ and a given profile $x$, and then computes an acceptance-maximizing decision for $x$. Let $\mathcal{A}$ be the set of all possible axioms and define the mechanism $M: \mathcal{A}^n \times \X \rightarrow \F$ that selects a single rule for every profile such that $f_x = M(A^*, x) := \argmax\limits_{f \in \F} \sum_{i \leq n} A_i(x,f,f(x))$.\textsuperscript{3} Here we run into counterfactual deception. Telling the agents that the decision is $(x, f_x, f_x(x))$ is a form of deception, because for profile $x'$ the decision would not be $(x', f_x, f_x(x'))$.

Consider instead the rule $f^*$ such that $f^*(x) = f_x(x)$ for all profiles, and suppose agents are told that the decision is $(x, f^*, f^*(x))$.
Unfortunately, a decision $(x, f^*, f^*(x))$ may not obey many of the axioms in $A^*$. Consider again $A^*$ characterized by $(B_1, F_2, F_3)$ above. The rule $f^*$ would be $f_2$ from Example~\ref{ex:running}, and would only ever make decisions accepted by one agent. If we were to use $f_1$ instead, then decisions will sometimes be accepted by two agents and otherwise be accepted by one, which is a Pareto improvement over using $f_2$ in terms of acceptance.

\cite{dietrich2005reach} proposes using the mechanism $M$ in a special case where each agent reports a single rule $f$, but it is assumed that the agents' normative positions are captured by forcing blackbox axioms being reported as implied rules $f_i$, such that $B_i = \{(x,y) : f_i(x)=y\}$. With only blackbox axioms there is no possibility of counterfactual deception, but if the agents' normative positions can belong to any other class of non-trivial axioms, then the acceptance maximizing procedure becomes vulnerable.

\subsection{Conflicting Explanations}
Part of the promise of automatically generating explanations and justifications for decisions is that different evaluations of a single decision can be constructed for individual agents based on their personal axioms, but this can lead to conflicting explanations. Again we will use the definitions in our running example, Example~\ref{ex:running}. Suppose that one agent's axiom is the exigent forcing axiom $A_L$ characterized by $L$ given in the example, while a second agent's axiom is the forcing procedural axiom characterized by $F = \{f_2\}$. Once a decision is made, we tell the first agent that $A_L$ was obeyed when making the decision, and tell the second agent that $f_2$ was used to make the decision. If the true decision was, say, $(x_1, f_1, y_1)$ then the second agent is deceived, but if the true decision was $(x_1, f_2, y_1)$ then the first agent is deceived. The counterfactual deception is not revealed to the agents if they can observe only $x$ and $y$. However, if agents can audit the rule, view the explanations given to other agents, or are provided a sequence of explanations that must be consistent across multiple decisions, then deception may be revealed. A more complex example of counterfactual deception from individualized axiomatic explanations can be found in~\cite{abramowitz2021deciding} using the axioms of Arrow's general impossibility theorem.

\subsection{Exigent Axiom Evaluation}
By examining what information is sufficient for evaluation, we reveal that there are two types of exigent axioms that cannot be easily disambiguated when looking only at partition sets. For some exigent axioms, knowing $x$ and $f$ is sufficient for evaluation without needing to compute $f(x)$, while for others we need to know $f(x)$ to evaluate whether a decision obeys the axiom. Consider the statement,``If Carl will be voting then we should use unanimity rule, otherwise we should use majority rule." To determine if a decision obeys the axiomatization of this statement, we need to know $f$ and whether $x$ contains Carl's vote, but do not need to  compute $f(x)$. This contrasts with exigent axioms that also depend on some knowledge of the outcome $y$. For example, ``We should use the Borda rule to vote on the best paper at the conference, unless I submit a paper and my paper loses under the Borda rule but would win under Plurality rule, in which case we should vote by Plurality."

%Concrete example: How much is known about profile when rule is chosen? That relates to transaprency and whether a ``system is biased." If rule was chosen with knowledge of future $x$, but claims to be justified by axiom $A$, then it may be deceptive and agents may no longer accept a decision.

% This limits the design of decision-making systems, .

% Let $\mathcal{A}$ be the set of all possible axioms, i.e., all possible partitions of $\D$, and define the mechanism $M: \mathcal{A}^n \times \X \rightarrow \F$ that selects a single rule for every profile such that $f_x = M(A^*, x) := \argmax\limits_{f \in \F} \sum\limits_i A_i(x,f,f(x))$. Suppose now that profile $x$ is observed, $M$ is used to make a decision, and afterwards agents are told that the decision was $(x, f_x, f_x(x))$. Is this deception? In a sense, yes, because had the profile $x$ be different, the rule $f_x$ could have been different, and so $f_x$ implies false counterfactuals. And yet, if we consider the implied rule $f^*$ such that $f^*(x) = f_x(x)$ for all $x$, $f^*$ may not satisfy any of the axioms in $A^*$.

\section{Intraprofile Axioms}\label{sec_intraprofile}
\citet[Chapter~14]{fishburn2015theory} defines intraprofile and interprofile axioms as part of a taxonomy of procedural axioms defined in a specific logical language. The common usage is that ``Intraprofile axioms have a particularly simple structure in that they only speak about conditions on outcomes `one profile at a time' "- \cite{schmidtlein2023voting} while \textit{interprofile} axioms require consideration of multiple profiles.

Consider the statement, ``The decision should be made using Plurality Rule." Computing the plurality winner(s) only requires knowledge of the given profile. To verify that the outcome of a decision is the plurality winner only requires one to know the profile and outcome. However, to tell whether the Plurality Rule was used to make the decision requires reasoning about what would have happened under every other possible profile, which is part of the definition of Plurality Rule. So how can a procedural axiom be intraprofile?

We can see that consideration of multiple profiles differs between making versus evaluating decisions, and any definition of intraprofile or interprofile procedural axioms must contend with the fact that rules are defined with respect to all possible profiles. A related problem was encountered by prior work that aimed to define an ``instance of an axiom"~\citep{schmidtlein2022voting}, but when axioms are seen as properties of decisions rather than properties of rules, the problem disappears because an ``instance" is just a decision. \cite{schmidtlein2023voting} offers a language-independent domain-specific definition of intraprofile axioms, which is a little roundabout. We extend their definition to be domain-agnostic below.

\begin{definition}[Intraprofile Axioms from \citeauthor{schmidtlein2023voting}]
    Procedural axiom $A$ is intraprofile if the procedural extension of its reduction is the original axiom $A$.
\end{definition}

Unlike evaluating procedural axioms, evaluating blackbox axioms requires only the singular profile and outcome of a decision.
% , and not the rule which contains all the implications about other profiles.
Here, we offer a simpler, equivalent definition of intraprofile axioms based on their relationship to blackbox axioms. This is possible because every axiom has only a single reduction and every blackbox axiom has a single procedural extension.

\begin{definition}[Intraprofile Axioms]\label{def:intraprofile_procedural}
    An axiom is intraprofile if and only if it is a procedural extension, i.e., $\exists B \subseteq \mathcal{X} \times \mathcal{Y}$ such that $F = \{f \in \mathcal{F} : (x, f(x)) \in B, \ \forall x \in \mathcal{X}\}$.
\end{definition}

Informally, ``The outcome must be Pareto-optimal" is a blackbox axiom  whose procedural extension is the intraprofile axiom ``The rule must be Pareto-efficient, producing a Pareto-optimal outcome for all profiles."
As another example, recall Example~\ref{ex:running}. Each procedural axiom characterized by a single rule $|F|=1$ is intraprofile, $F = \{f_1,f_2\}$ is intraprofile, and $F = \F$ is intraprofile, but $F = \{f_2, f_3\}$ is not intraprofile (i.e., interprofile). For $F = \{f_2, f_3\}$ a blackbox axiom $A_B$ whose extension is $F$ would need to be characterized by $B = \{(x_1, y_1)$, $(x_2, y_1)$, $(x_1, y_2), (x_2, y_2)\}$, but the procedural extension of $A_B$ is $\F = \{f_1, f_2, f_3, f_4\}$, not $\{f_2, f_3\}$.

\subsection{Intraprofile Axioms in Voting by Axioms}
\cite{schmidtlein2023voting} state that if, ``an [intraprofile] axiom forces outcomes (possibly jointly with other axioms), not only will every forced outcome be consistent with the axiom but the procedure (across all such profiles) will be as well." We prove this is mistaken via counterexample.% with intraprofile procedural axioms.

\begin{example}\label{ex:VBA_incorrect}
Let $\mathcal{X} = \{x_1, x_2\}$, $\mathcal{Y} = \{y_1, y_2, y_3\}$, and consider the following subset of rules:

$$\begin{cases} 
    f_1 := f_1(x_1) = y_1, & f_1(x_2) = y_1 \\
    f_2 := f_2(x_1) = y_1, & f_2(x_2) = y_2 \\
    f_3 := f_3(x_1) = y_2, & f_3(x_2) = y_3 \\
    f_4 := f_4(x_1) = y_3, & f_4(x_2) = y_3 \\
    f_5 := f_5(x_1) = y_1, & f_5(x_2) = y_3 \\
\end{cases}$$

Let $\vec{A} = (A_1, A_2)$ be the vector of procedural axioms characterized by $F_1 = \{f_1, f_2\}$ and $F_2 = \{f_3, f_4\}$ respectively.
\end{example}

$A_1$ forces $y_1$ on $x_1$ but is not forcing $x_2$, while $A_2$ forces $y_3$ on $x_2$ but is not forcing on $x_1$, and both $A_1$ and $A_2$ are intraprofile. The implied rule of Voting by Axioms is $f_5$, and yet decisions made by $f_5$ will not satisfy either axiom. Thus, the Decision-Evaluation Paradox arises in Voting by Axioms even when all axioms are intraprofile. Similarly $OR(\vec{A})$ and $AND(\vec{A})$ may not be intraprofile even when all the constituent axioms are, which also follows from Example~\ref{ex:VBA_incorrect}.

\section{Conclusions}
We have proposed Axiomatic Choice, a domain-agnostic framework for making and evaluating decisions using axioms to encode normative positions, where axioms are defined as properties of decisions. We have focused on establishing the fundamental challenges that hold across all decision domains, including the Decision-Evaluation Paradox and vulnerability to deception, and showed how these challenges depend on the taxonomic class of the axiom being considered.  

We generalize properties of procedural axioms in the existing literature to make them domain-agnostic and apply to all axioms, and use them to derive the Decision-Evaluation Paradox. Our paradox places a fundamental limit on the use and transparency of axiom-based decision-making procedures, including Voting by Axioms from~\cite{schmidtlein2023voting} the method rooted in ``procedural autonomy" proposed by~\cite{dietrich2005reach}. A common concern is whether a decision is sufficiently transparent, which depends in part on whether the underlying procedure is auditable. We shed light on the nature of transparency and deception in explaining and justifying decisions, particularly when such decision evaluations are individualized, and show how the decision-evaluation paradox can lead to deception. Finally, we build on~\cite{schmidtlein2023voting} to provide a minimal, domain-agnostic, language-independent definition of intraprofile axioms, and show that such axioms do not evade the Decision-Evaluation Paradox in Voting by Axioms.

\subsection{Future Work}
Our domain-agnostic framework can be extended to consider orderings over decisions, which generalizes the orderings over orderings proposed by~\cite{sen2017collective}. It can be further extended by assigning cardinal values to decisions, capturing path-dependence where an agent's utility depends on how a decision gets made.

For any specific domain, our taxonomy may be further refined by what information is necessary and sufficient for evaluation. Similarly, there are open questions around how to elicitation normative positions, interpret natural language statements as axioms, and the complexity of decision-making and evaluation in any given domain. Operationalizing axioms is an additional challenge when $\D$ is large or infinite. Finally, the identification of counterfactual deception depends on the causal model one chooses, which will also tend to be domain-specific or application-specific, and may differ between agents.

\pagebreak
\appendix

\section*{Ethical Statement}

There are no ethical issues.

% \section*{Acknowledgments}
% \ben{Nick To Do}

\bibliographystyle{named}
\bibliography{ijcai26}

\end{document}